\documentclass{article}
\pdfoutput=1

\usepackage[preprint,nonatbib]{neurips_2021}

\usepackage[utf8]{inputenc} 
\usepackage[T1]{fontenc}    
\usepackage{amsmath} 
\usepackage{amsthm,multirow}
\usepackage{hyperref}       
\usepackage{url}            
\usepackage{booktabs}       
\usepackage{amsfonts}       
\usepackage{nicefrac}       
\usepackage{microtype}      
\usepackage{xcolor}         
\usepackage{graphicx}
\usepackage{caption}
\usepackage{subcaption}
\usepackage{epstopdf}

\def \grad {\nabla}
\def \R {\mathbb{R}}
\def \N {\mathbb{N}}
\def \E {\mathbb{E}}
\def \O {\mathcal{O}}
\def \T {\mathcal{T}}
\def \calF {\mathcal{F}}

\def \D {\mathrm{D}}
\def \input {\mathrm{in}}
\def \output {\mathrm{out}}

\newcommand{\tr}[1]{\mathrm{tr}\left(#1 \right)}
\newcommand{\expm}[1]{\mathrm{expm}\left(#1 \right)}
\newcommand{\nbf}[1] {{#1}}
\def \Exp {\mathrm{Exp}}
\def \argmax {\mathrm{argmax}}

\usepackage{enumitem}
\newtheorem{theorem}{Theorem}
\newtheorem{assumption}{Assumption}
\usepackage{algorithm}
\usepackage{algpseudocode}

\title{Coordinate descent on the orthogonal group for recurrent neural network training}

\author{%
  Estelle Massart\thanks{Second affiliation: National Physical Laboratory, Hampton Road, Teddington, Middlesex, TW11 0LW, UK}\\
  Mathematical Institute\\
  University of Oxford\\
  Oxford, OX2 6GG, UK \\
  \texttt{massart@maths.ox.ac.uk} \\
   \And
   Vinayak Abrol \\
   IIIT Delhi \\
   India \\
   \texttt{ abrol@iiitd.ac.in} \\
}

\begin{document}

\maketitle

\begin{abstract}
We propose to use stochastic Riemannian coordinate descent on the orthogonal group for recurrent neural network training. The algorithm rotates successively two columns of the recurrent matrix, an operation that can be efficiently implemented as a multiplication by a Givens matrix. In the case when the coordinate is selected uniformly at random at each iteration, we prove the convergence of the proposed algorithm under standard assumptions on the loss function, stepsize and minibatch noise. In addition, we numerically demonstrate that the Riemannian gradient in recurrent neural network training has an approximately sparse structure. Leveraging this observation, we propose a faster variant of the proposed algorithm that relies on the Gauss-Southwell rule. Experiments on a benchmark recurrent neural network training problem are presented to demonstrate the effectiveness of the proposed algorithm. 
\end{abstract}

\section{Introduction}
Exploding or vanishing gradients are key issues affecting the training of deep neural networks (DNNs), and are particularly problematic when training recurrent neural networks (RNNs) \cite{Bengio1994,Pascanu2013}, an architecture that endows the network with some memory and has been proposed for modeling sequential data (see, e.g., \cite{Giles94}). In recurrent neural networks, the signal propagation is described by the following pair of equations
\begin{equation} \label{eq:network}
    \left\{ \begin{array}{rcl}
    h(t+1) & = &\phi(W_{\input} x(t+1) + W h(t)), \\ 
    y(t+1) & = & W_{\output} h(t+1) + b_{\output},
    \end{array} \right.
\end{equation}
for $t = 0, 1, \dots$, with input $x(t) \in \mathbb{R}^{d_{\input}}$, hidden state $h(t) \in \mathbb{R}^{d}$, and output $y(t) \in \mathbb{R}^{d_{\output}}$. The mapping $\phi: \R\rightarrow \R$ is a pointwise nonlinearity, and $W \in \mathbb{R}^{d\times d}$, $W_{\input} \in \mathbb{R}^{d\times d_{\input}}$, $W_{\output} \in \mathbb{R}^{d_{\output}\times d}$  and $b_{\output} \in \mathbb{R}^{d_{\output}}$ are model parameters. The repeated multiplication of the hidden state by the recurrent matrix $W$ makes these architectures particularly sensitive to exploding or vanishing gradients.

A recently proposed remedy against exploding and vanishing gradients in RNNs imposes the recurrent weight matrix $W$ to be orthogonal/unitary, see, e.g., \cite{Arjovsky2016}. However, for large models, enforcing this constraint (e.g., by projecting the matrix on the set of orthogonal/unitary matrices, at each iteration) comes with substantial computational costs, scaling cubically with $d$. Several solutions have been proposed to alleviate these costs, and are summarized in Section~\ref{sec:background}. In this work, using the fact that the set of orthogonal matrices admits a Riemannian manifold structure, we propose a stochastic Riemannian coordinate descent algorithm for RNN training. The resulting algorithm works as follows: for each mini-batch, instead of updating the full matrix $W$, we apply a rotation to a pair of columns of $W$ only; this is equivalent to restricting the update of the matrix $W$ to one coordinate of the tangent space to the manifold per iteration (for a suitably chosen basis of that tangent space). This rotation can be implemented efficiently as a multiplication by a Givens matrix, and the resulting cost per iteration scales linearly with $d$. 

\subsection{Related works}
\label{sec:background}
\paragraph{Orthogonal/unitary RNNs.}  Unitary RNNs have been initially proposed in \cite{Arjovsky2016} to avoid exploding or vanishing gradients. Various algorithms have then been developed to alleviate the additional computational burden implied by repeated orthogonalization, mostly proposing different parametrizations of the orthogonal/unitary group, see \cite{Arjovsky2016,Wisdom2016,Jing2017,Hyland2017,Mhammedi2017,Helfrich2018, Maduranga2019, Lezcano2019}. However, most of these works come with no theoretical guarantee to recover a stationary point of the original problem. Alternative approaches build on the geometry of the set of orthogonal/unitary matrices: \cite{Wisdom2016} propose a stochastic Riemannian gradient descent on the manifold of unitary matrices combined with the Cayley retraction, while the algorithm given in \cite{Lezcano2019} can be seen as a stochastic Riemannian gradient descent on the orthogonal group using a different Riemannian metric, relying on the Lie group structure of those sets. \cite{Lezcano2019} also gives an implementation trick relying on Pad\'e approximants for reducing the cost of the matrix exponential (and its gradient). Let us finally mention two recent toolboxes extending the PyTorch class \emph{torch.optim} to parameters constrained to lie on some manifolds, including the orthogonal group: McTorch \cite{Mctorch} and GeoTorch \cite{lezcano2019trivializations}. 

\paragraph{Orthogonal constraints in other DNN architectures.} Orthogonal weights have also been used for other network architectures, including fully-connected, convolutional and residual neural networks \cite{Ozay2016,Huang2018,Huang2020,Bansal2018,Jia2020}. Using orthogonal weight matrices in DNNs preserves the energy of the signal propagated through the hidden units, and has been shown numerically to result in a lower test error than comparable unconstrained DNNs \cite{Huang2018,Jia2020}. Orthogonal initialization moreover leads to dynamical isometry \cite{Pennington2017}, the desirable regime in which the spectrum of the input-output Jacobian is concentrated around the value one, shown to result in a faster training, see \cite{Pennington2017,Hu2020,murray2021}. For arbitrary DNN architectures, orthogonal regularization is sometimes preferred over strict orthogonal constraints (see, e.g., \cite{Jia2017,Yoshida2017,Bansal2018}).

\paragraph{Coordinate descent algorithms}
In the Euclidean setting, coordinate descent (CD) is a well-studied algorithm that minimizes successively the loss along a coordinate, typically using a coordinate-wise variant of a first-order method, and has been shown to achieve state-of-the-art results on a range of high-dimensional problems \cite{Wright2015}. In particular, CD algorithms have recently been applied to DNN training, see, e.g., \cite{Zheng2019,Palagi2019} and reference therein. Though the convergence of CD has been widely addressed in the literature \cite{Patrascu2014,Wright2015,Zheng2019},  few works explore the convergence of stochastic/mini-batch CD, i.e., address the situation where only a stochastic estimate of the partial derivatives is available. Exceptions include \cite{Wang2014,Zhao2014}, which assumes strong convexity of the objective, and \cite{Palagi2021} which addresses  block-coordinate incremental (i.e., deterministic) gradient descent in the nonconvex setting. On the orthogonal group, CD has been applied in \cite{Ishteva2013} for low (multilinear) rank approximation of tensors, and in \cite{Shalit2014} for sparse PCA and tensor decomposition. Note that these algorithms minimize exactly the loss along the coordinate direction at each iteration, an operation which is out of reach in DNN training. The convergence of coordinate descent on manifolds, using coordinate-wise first-order updates relying on exact partial derivatives, has been very recently studied  in \cite{Gutman2020}. 

\subsection{Contributions}
We propose a stochastic Riemannian coordinate descent (SRCD) algorithm for training orthogonal RNNs, in which the cost per iteration scales linearly with the size $d$ of the recurrent matrix $W$. The proposed algorithm is close to the one presented in \cite{Jing2017}; we highlight three main differences. First, \cite{Jing2017} addresses unitary RNNs, while we focus on orthogonal RNNs. Our choice to consider real-valued models is motivated by the fact that, from the viewpoint of model representation, it is equivalent to work with a unitary RNN or with an orthogonal RNN with a twice bigger hidden matrix \cite{Mhammedi2017}. Secondly, the update rule for the recurrent matrix in \cite{Jing2017} involves the multiplication by a diagonal matrix, which does not appear in our model. Finally, taking here the viewpoint of optimization on manifolds allows us to derive convergence guarantees for SRCD, while no counterparts of those results are given in \cite{Jing2017}. More precisely, we prove the convergence of SRCD under standard assumptions on the (mini-batch) gradient noise, stepsize and objective, for coordinates selected uniformly at random at each iteration. As a second contribution, we show numerically that the Riemannian gradient of the loss (with respect to the orthogonal parameter) has an approximately sparse representation in the basis considered for the tangent space, and propose a variant of SRCD in which the coordinate is selected using the Gauss-Southwell rule at each iteration. This is to our knowledge the first application of this rule to optimization problems defined on manifolds. We finally illustrate numerically the behavior of the proposed algorithms on a benchmark problem. Implementations of the proposed algorithms can be found at \url{https://github.com/EMassart/OrthCDforRNNs}.


\section{Stochastic coordinate descent on the orthogonal group} \label{sec:algo}
In this paper, without loss of generality, we address the optimization problem
\begin{equation} \label{eq: optiprob} \tag{P}
    \min_{X \in \R^{m \times n}, W \in \O_d} f(X,W),
\end{equation} 
where $\mathcal{O}_d := \{ W \in\mathbb{R}^{d \times d} : W^\top W = I_d\}$ is the set of $d \times d$ orthogonal matrices. For the specific framework of RNN training, the variable $X$ refers to the parameters $W_{\input}$, $W_{\output}$, $b_{\output}$ or any other unconstrained model parameter, while $W$ is the recurrent matrix. We assume throughout the paper that the dimension of $W$ is very large. Before introducing our algorithm, let us first summarize properties of the set of orthogonal matrices that we use in the rest of the paper.

\subsection{Geometry of the manifold of orthogonal matrices} \label{sec:geometry}
The set $\mathcal{O}_d$ of orthogonal matrices is a Riemannian manifold. Roughly speaking, a manifold is a topological set in which each neighbourhood can be set in one-to-one correspondence with an open set of the Euclidean space; in that sense, manifolds are sets that ``locally look flat'' (see \cite{AMS2008} or \cite{boumal2020intromanifolds} for more a formal definition). Typically, computations on manifolds are mostly done on the \emph{tangent space} of the manifold at a given point, a first-order approximation of the manifold around that point. For the orthogonal group, the tangent space $\mathcal{T}_W \mathcal{O}_d$  at some point $W \in \mathcal{O}_d$ is given by:
\begin{equation} \label{eq: tgSpace}
    \mathcal{T}_{W} \mathcal{O}_d = \{ W \Omega : \Omega \in \R^{d \times d}, \Omega = - \Omega^\top \}.
\end{equation}
This is a $D$-dimensional vector space, with $D = d(d-1)/2$; consistently, $\O_d$ is a $D$-dimensional manifold. Riemannian manifolds are manifolds whose tangent spaces are endowed with a Riemannian metric (a smoothly varying inner product, providing tangent vectors with a notion of length). As often for the manifold of orthogonal matrices \cite{AMS2008}, we choose as Riemannian metric the classical Euclidean metric: for each tangent space $\mathcal{T}_W \mathcal{O}_d$, and for all $\xi_W, \zeta_W \in \mathcal{T}_W \mathcal{O}_d$, the Riemannian metric between $\xi_W$ and $\zeta_W$ is given by: 
\begin{equation} \label{eq:metric}
    \langle \xi_{W}, \zeta_W\rangle := \tr{\xi_W^\top \zeta_W},
\end{equation}
the usual Frobenius inner product. In this paper, we therefore use the notation $\langle \cdot, \cdot \rangle$ to refer both to the Riemannian metric, and the Euclidean inner product. Considering this metric, it can easily be checked that the tangent space $\mathcal{T}_W \O_d$ is generated by the orthonormal basis \cite{Shalit2014}:
\begin{equation} 
\label{eq: basis}
\begin{aligned}
&\mathcal{B}_W := \{ \eta_i\}_{i = 1}^{D}, \text{where} \ \eta_{i} := W H_{j,l},
\end{aligned}
\end{equation}
with $j,l$ two indices\footnote{The following simply characterizes our numbering of the basis elements: the first coordinate vector is associated to the pair $(j = 1, l = 2)$, the second to the pair $(j = 1, l = 3)$ etc until $(j = d-1, l = d)$.} such that $1 \leq j < l \leq d$ and $i = \sum_{k = 1}^{j-1}(d-k) + (l-j)$, and where the set of matrices $\{H_{j,l}\}$, with $1 \leq j < l \leq d$ is an orthonormal basis for the vector space of $d \times d$ skew-symmetric matrices: 
\begin{equation} \label{eq: H}
    H_{j,l}  := \frac{1}{\sqrt{2}} \left( e_je_l^\top  - e_l e_j^\top \right),
\end{equation}
with $e_j \in \mathbb{R}^d$ the vector whose elements are all zero, except the $j$th component that is equal to one. Let us emphasize that each tangent space is thus endowed with a norm function, and that this norm function coincides with the Frobenius norm:
\begin{equation} \label{eq:norm}
    \| \xi_{W} \| =  \langle \xi_{W}, \xi_{W} \rangle^{\frac{1}{2}} = (\tr{\xi_W^\top \xi_W})^{\frac{1}{2}}.
\end{equation}
 The Riemannian gradient is the counterpart of the Euclidean gradient on manifolds. Given an arbitrary function $h : \mathcal{O}_d \to \mathbb{R}$, the Riemannian gradient of $h$ at $W \in \mathcal{O}_d$ is the unique vector $\grad h(W) \in \mathcal{T}_W \mathcal{O}_d$ that satisfies:
\[ \langle \xi_W, \grad h(W)\rangle  = \D h(W)[\xi_W] \qquad \forall \xi_W \in \T_{W} \O_d, \]
i.e., its inner product with any tangent vector $\xi_W$ gives the directional derivative of $h$ along the direction spanned by the tangent vector $\xi_W$, written here $\D h(W)[\xi_W]$. On the orthogonal group, the Riemannian gradient is simply computed\footnote{This follows from the fact that the orthogonal group is an embedded submanifold of $\R^{d \times d}$, see \cite[Chap. 3]{AMS2008} for more information.} as
\[ \grad h(W) = P_{\T_W \O_d} (\nabla^e h(W)), \]
where $\nabla^e  h(W)$ is the Euclidean gradient of $h$, where $h$ is seen as a function from $\R^{d \times d}$ to $\R$, and $P_{\T_W \O_d}(\cdot)$ is the orthogonal projection on the tangent space $\T_W \O_d$:
\begin{equation}
    P_{\T_W \O_d}(M) = W \left(\frac{W^\top M - M^\top W}{2}\right)
\end{equation}
for all $M \in \R^{d \times d}$. Given the basis \eqref{eq: basis}, one has a Riemannian counterpart to the notion of partial derivative. We define the $i$th Riemannian partial derivative of $h$, for the basis \eqref{eq: basis} of the tangent space $\mathcal{T}_W \O_d$, by:
\begin{equation} \label{eq: partial}
    \grad_{i} h(W) = \langle \grad h(W), \eta_i \rangle,
\end{equation}
where $\eta_i$ is the $i$th coordinate vector of the basis \eqref{eq: basis}. This Riemannian partial derivative can be obtained easily from the Euclidean gradient $\nabla^e h(W)$ (typically computed using backpropagation in DNN training):
\begin{equation}\label{eq:partialderiv}
   \grad_{i} h(W)  = \tr{H_{j,l}^\top W^{\top} \grad h(W)} = \tr{H_{j,l}^\top W^{\top} \nabla^e h(W)},
\end{equation}
where $j, l \in \{1, \dots, d\}$ are defined after \eqref{eq: basis}, and where the second equality comes from the fact that the Frobenius inner product of a skew-symmetric matrix ($H_{j,l}$) and a symmetric matrix ($W^{\top} (\grad h(W) - \nabla^e h(W)))$ is zero.

A last tool that is required in this work is the exponential map, an operator that allows to move on the manifold in a given direction. The exponential map $\mathrm{Exp}_W (\xi_W)$ returns the point obtained by evaluating at time $t = 1$ the geodesic (curve with no acceleration, i.e., straight lines in the Euclidean space), starting at  $W$, with initial velocity $\xi_W$. For the manifold of orthogonal matrices, the exponential map is given by
\begin{equation} \label{eq:exp}
    \mathrm{Exp}_W (\xi_W) = W \; \expm{ W^\top \xi_W}, \; \forall \xi_W \in \mathcal{T}_W \mathcal{O}_d,
\end{equation}
with $\expm{}$ the matrix exponential. Figure \ref{fig:opti_manifold} provides an abstract representation of a Riemannian gradient descent iteration.

\begin{figure}
    \centering
    \includegraphics[scale = 0.3, trim = 0 20 0 50, clip]{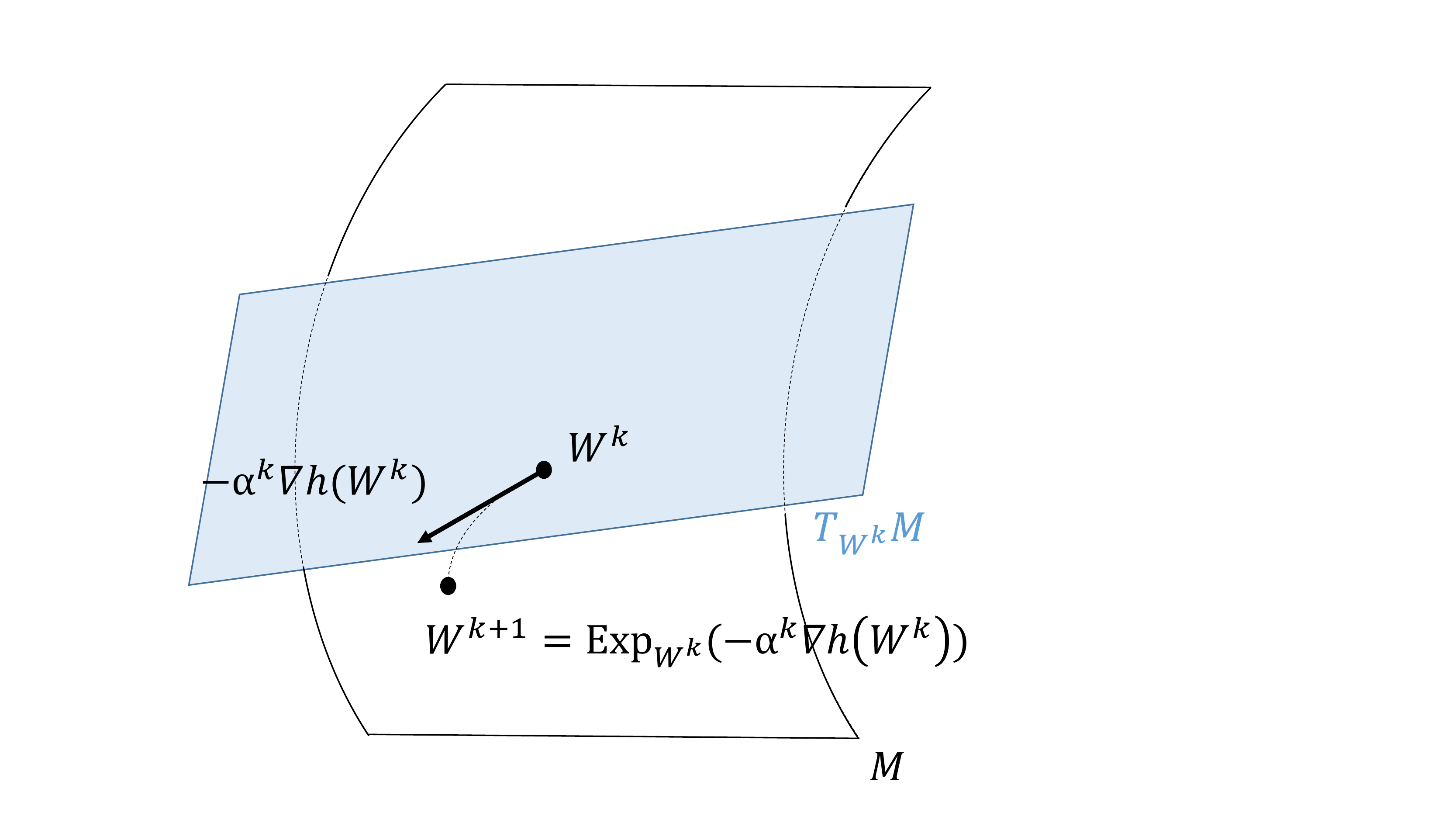}
    \caption{Graphical illustration of a Riemannian gradient descent iteration on a manifold $M$. At each iteration, we first compute the Riemannian gradient at the current iterate $W^k$ and then update the iterate by ``moving on the manifold in the direction opposite to the gradient'', using the exponential map. }
    \label{fig:opti_manifold}
\end{figure}

\subsection{A brief overview of Stochastic Riemannian Gradient Descent (SRGD)}
Building on the previous discussion, we now recall in Algorithm \ref{alg: RGD} the celebrated stochastic Riemannian gradient descent algorithm, initially proposed and analysed in \cite{Bonnabel2013}, that we apply here to \eqref{eq: optiprob}. We use this algorithm as a comparison point throughout this paper. Slightly departing from the notation introduced above, and in order to simplify the notation, we use hereafter the notation $g_X^k$ and $g_W^k$ to refer respectively to the gradient of $f$ with respect to the unconstrained variable $X$ and the (Riemannian) gradient of $f$ with respect to $W$ at iteration $k$, and $\tilde g_X^k$ and $\tilde g_W^k$ for their stochastic counterparts. In other words, the full exact and stochastic (Riemannian) gradient of $f$ are simply
\[  g^k = \begin{pmatrix}
g_X^k \\
g_W^k \\
\end{pmatrix} \qquad \text{and} \qquad \tilde{g}^k = \begin{pmatrix}
\tilde{g}_X^k \\
\tilde{g}_W^k \\
\end{pmatrix}.\]
In the case when $f$ is a sum of a large number of functions, as in our RNN training application, $\tilde{g}_{X}^k$ and $\tilde{g}_{W}^k$ can be seen as the approximations of $g_X^k$ and $g_W^k$ computed over a mini-batch. At each iteration, the stochastic gradients $\tilde{g}_X^k$ and $\tilde{g}_W^k$ are computed, and the iterates $X^k$, $W^k$ are updated. Note that we evaluate the exponential map \eqref{eq:exp} at each iteration, involving a matrix exponential whose cost evolves cubically with $d$ (PyTorch matrix exponential implementation currently relies on Taylor/ Pad\'e approximants and the scaling-squaring trick, see \cite{Bader2019}). 

\begin{algorithm}[t]
	\caption{SRGD: Stochastic Riemannian Gradient Descent}
	\label{alg: RGD}
	\begin{algorithmic}[1]
		\State Let $\{\alpha^k\}$ be a sequence of stepsizes. Set $k = 0$, and initialize the unconstrained and orthogonal variables $X^0 \in \R^{m \times n}, W^0 \in \mathcal{O}_{d}$. 
		\While{not converged} 
		\State Compute the (stochastic) gradients $\tilde{g}_X^k$ and $\tilde{g}_W^k$
		\State Update the unconstrained variable: $X^{k+1} = X^k - \alpha^k \tilde{g}_X^k$
		\State Update the constrained variable: $W^{k+1} =  \mathrm{Exp}_{W^k} (-\alpha^k \tilde{g}_{W}^k)$
		\State $k := k+1$
		\EndWhile
	\end{algorithmic}
\end{algorithm}

\subsection{Proposed algorithm: Stochastic Riemannian Coordinate Descent (SRCD)}
Algorithm \ref{alg: RCD} presents our proposed SRCD algorithm. At each iteration, the unconstrained parameters are updated using stochastic gradient descent, exactly as in Algorithm \ref{alg: RGD}; the modification is in the update rule for the orthogonal parameter. Here, the iterate is only updated along one coordinate of the tangent space, for the basis \eqref{eq: basis}. This amounts to rotating a pair of columns of the iterate $W^k$ of an angle that depends on the stepsize and the component of the gradient along that coordinate. The stochastic gradient $\tilde{g}_W^k$ is thus replaced by the following tangent vector, which is aligned with the $i^k$th coordinate vector $\eta_{i^k}$ of the tangent space $\mathcal{T}_{W^k} \O_d$: 
\begin{equation} \label{eq: tildegcoord}
    \tilde g_{W,i^k}^k = \langle \tilde{g}_W^k, \eta_{i^k} \rangle \eta_{i^k},
\end{equation} 
where $\eta_{i^k} = W^k H_{j^k,l^k}$ for some given $j^k,l^k$ defined below \eqref{eq: basis}. 

\begin{algorithm}[t]
	\caption{SRCD: Stochastic Riemannian Coordinate Descent}
	\label{alg: RCD}
	\begin{algorithmic}[1]
		\State Let $\{\alpha^k\}$ be a sequence of stepsizes. Set $k = 0$, and initialize the unconstrained and orthogonal variables $X^0 \in \R^{m \times n}, W^0 \in \mathcal{O}_{d}$. 
		\While{not converged} 
		\State Compute the (stochastic) gradient $\tilde{g}_X^k$
		\State Update the unconstrained variable: $X^{k+1} = X^k - \alpha^k \tilde{g}_X^k$
		\State Select a coordinate $i^k \in \{1, \dots, D\}$ of the tangent space $\mathcal{T}_{W^k}(\O_d) $
		\State Compute the (stochastic) partial derivative $\tilde{g}_{W, i_k}^k$ using \eqref{eq:partialderiv} \label{algline:partialderiv}
		\State Update the orthogonal variable: $W^{k+1} =  \Exp_{W^k}( - \alpha^k \tilde{g}_{W, i_k}^k)$. Due to the structure of $\tilde{g}_{W, i_k}^k$, this step simply involves the multiplication of $W^{k}$ by a Givens matrix, see \eqref{eq: givensupdate}, and has a computational cost evolving linearly with $d$. \label{line:coordinateStep}
		\State $k := k+1$
		\EndWhile
	\end{algorithmic}
\end{algorithm}

Note that, since the matrix $H_{j^k, l^k}$ has a very special structure, see \eqref{eq: H}, the exponential map in Line \ref{line:coordinateStep} can be written as a multiplication by a Givens matrix. Indeed, writing $\theta^k = \langle \tilde{g}_W^k, \eta_{i^k} \rangle$, there holds 
\begin{equation} \label{eq: givensupdate}
    \mathrm{Exp}_{W^k}(\theta^k \eta_{i^k}) = \mathrm{Exp}_{W^k}(\theta^k W^k H_{j^k,l^k}) = W^k G_{j^k,l^k}(\theta^k),
\end{equation}
with 
\begin{equation}
    G_{j^k,l^k}(\theta^k) := \begin{pmatrix}  
    1 &\cdots &0 &\cdots &0 &\cdots &0 \\  
    \vdots & \ddots & \vdots &\vdots & \vdots &\vdots & \vdots \\
    0 &  \cdots & \cos(\theta^k)& \cdots  & \sin(\theta^k) &\vdots & 0 \\
    \vdots &  & \vdots & \ddots  & \vdots  & & \vdots \\
    0 &  \cdots & -\sin(\theta^k)& \cdots  & \cos(\theta^k) &\vdots & 0 \\
    \vdots &   & \vdots &  & \vdots &\ddots & \vdots \\
    0 &  \cdots & 0& \cdots  & 0 &\cdots & 1 \\
    \end{pmatrix}
\end{equation}
a Givens matrix \cite{Shalit2014}. Right-multiplying the iterate $W^k$ with $G_{j^k,l^k}(\theta^k)$ has the effect of rotating clockwise the $j^k$th and $l^k$th columns of $W^k$, with rotation angle $\theta^k$. In particular, since Givens matrices are sparse (having only 4 nonzero elements), the evaluation of the right-hand side of \eqref{eq: givensupdate} has a cost of about $6d$ flops, which is much lower than the cost of the matrix exponential that is used when updating the $W$ variable in SRGD.

Algorithm \ref{alg: RCD} also requires a strategy for selecting the coordinate at each iteration. Two well-known strategies in the Euclidean setting are uniform sampling, where $i^k$ is sampled independently and uniformly among $\{1, \dots, D\}$ at each iteration (with $D = d(d-1)/2$ the dimension of the manifold), and the Gauss-Southwell rule, where $i^k$ is the coordinate associated to the fastest local variation of the objective:  
\begin{equation} \label{eq: GS}
    i^k = \underset{i \in \{1, \dots, D\}}{ \argmax} \| g_{W,i}^k \|.
\end{equation}
We compare both strategies numerically in Section \ref{sec:numerics}.

\subsection{Convergence analysis} \label{sec:convergence}
Our convergence analysis heavily relies on the convergence analysis of stochastic gradient descent in the Euclidean setting (see \cite[Chap. 4]{Bottou2018}). The main contribution of the convergence analysis is a careful extension to the Riemannian setting, and to coordinates updates for one of the variables. First, let us introduce the following smoothness assumption on the function $f$, following \cite{Boumal2018}. 

\begin{assumption} \label{ass:smooth}
The function $f : \R^{m \times n} \times \O_d \to \R$ is $L$-smooth, i.e., satisfies for all $(X,W) \in \R^{m \times n} \times \O_d$ and $(\nu, \mu) \in \R^{m \times n} \times \T_{W}\O_d$:
\[ \left| f(X+\nu,\Exp_{W}(\mu)) - f(X,W) - \langle g_X(X,W), \nu \rangle  - \langle g_W(X,W), \mu \rangle \right| \leq \frac{L}{2} \left( \|\nu\|^2 + \| \mu \|^2\right). \]
\end{assumption}

Assumption \ref{ass:smooth} is guaranteed to be satisfied if the Euclidean gradient of the function $f$ (when the latter is seen as a function on $\R^{m \times n} \times \R^{d \times d}$) is Lipshitz-continuous, with Lipshitz constant $L$. This follows from Lemma 7 in \cite{Boumal2018}, noticing that the proof of that lemma still holds when the manifold is a product manifold of a compact manifold (here the orthogonal group) and a Euclidean space\footnote{Actually, we just need the gradient of the function to be Lipshitz continuous on $\R^{m \times n} \times \mathrm{Conv}(\mathcal{O}_d) \subset \R^{m \times n} \times \R^{d \times d}$, where $\mathrm{Conv}(\mathcal{O}_d)$ is the convex hull of $\O_d$.}.

Our second assumption is classical when analysing stochastic gradient descent in the Euclidean setting, see, e.g., \cite[Chap. 4]{Bottou2018}. Let us write hereafter $\calF^k = \{ \tilde g_X^0, \tilde g_W^0, i^0, \dots, \tilde g_X^{k-1}, \tilde g_W^{k-1}, i^{k-1}\}$, the $\sigma$-algebra generated by the random variables before iteration $k$.

\begin{assumption} \label{ass:approxGradient}
The gradients and iterates of the algorithm satisfy the conditions:
\begin{enumerate}
    \item The sequence of iterates $\{ (X^k, W^k)\}_{k \in \mathbb{N}}$ is contained in an open set over which the value of the function $f$ is lower bounded by some $f_{\inf}$,
    \item There exist $\mu_X, \mu_W > 0$ such that for all $k \in \mathbb{N}$ and $Z \in \{X, W\}$,
    \[ \langle g_Z^k, \E \left[ \tilde{g}_Z^k | \calF^{k} \right] \rangle \geq \mu_Z \| g_Z^k \|^2.\]
    \item There exists $C_X, C_W \geq 0$ and $M_X, M_W \geq 0$ such that, for all $k \in \N$ and for $Z \in \{X, W\}$,
    \[\E \left[ \| \tilde{g}_Z^k\|^2 | \calF^{k}\right] \leq C_Z + M_Z \| g_Z^k\|^2.\]
\end{enumerate}
\end{assumption}

Under these assumptions, we prove the following result, which is analogous to \cite[Thm. 4.10]{Bottou2018}.
\begin{theorem}[Convergence result]
Under Assumptions \ref{ass:smooth} and \ref{ass:approxGradient}, the sequence of iterates $\{(X^k, W^k)\}$ generated by Algorithm \ref{alg: RCD}, with coordinate $i^k$ selected for each $k$ uniformly at random among $\{1, \dots, D\}$, with $D = d(d-1)/2$ the dimension of $\O_d$,  and using a sequence of stepsizes $\alpha^k$ that satisfies the Robbin-Monro conditions \cite{Robbins1951}:
\begin{equation} \label{eq: step}
    \lim_{k \to \infty} \sum_{i = 0}^k \alpha^k = \infty \qquad \lim_{k \to \infty} \sum_{i = 0}^k (\alpha^k)^2 = 0
\end{equation}
satisfies 
\begin{align}
    \mathbb{E} \left[ \frac{1}{\sum_{k = 0}^K \alpha^k} \sum_{k = 0}^K \alpha^k \| g^k \rVert^2 \right] \to 0 \qquad \text{as} \ K \to \infty,
\end{align}
where $g^k = [ g_X^{k \top} g_W^{k \top}]^\top$ is the (Riemannian) gradient of the objective at iterate $(X^k,W^k)$.
\end{theorem}
\begin{proof}
See supplementary material.
\end{proof}

\section{Numerical experiments} \label{sec:numerics}


We consider in this paper the copying memory task, which consists in remembering a sequence of letters from some alphabet $\mathcal{A} = \{ a_k\}_{k = 1}^N$. The model is given a sequence

\[{\large (\underbrace{a_{s_1}, a_{s_2}, \dots, a_{s_K}}_{\text{Sequence to remember}}, \underbrace{b, b, \dots, b}_{\text{$L$ blank symbols}}, \underbrace{c}_{\text{Start symbol}}, \underbrace{b, \dots, b}_{\text{$K$-1 blank symbols}}),}\]
with $a_{s_i} \in \mathcal{A}$ for all $i \in \{1, \dots, K\}$, and $b$, $c$ two additional letters that do not belong to the alphabet (respectively the ``blank'' and ``start'' symbols). The whole input sequence has total length $L+2K$; the first $K$ elements are the elements to remember. Once the model reads the ``start'' symbol, it should start replicating the sequence it has memorized. For the input sequence given above, the output of the network should thus be the sequence $(b, \dots, b, a_{s_1}, a_{s_2}, \dots, a_{s_K})$, made of $L+K$ replications of the blank letter followed by the sequence the model was asked to memorize. We rely on \url{https://github.com/Lezcano/expRNN} for the PyTorch implementation of the model architecture and of the copying memory task. In particular, the alphabet $\mathcal{A}$ comprises of 9 different elements, $L = 1000$, $K = 10$, the batchsize is equal to 128, and the recurrent matrix is an orthogonal matrix of size $190 \times 190$. The loss for this problem is the cross-entropy.


\begin{figure}[t!]
    \centering
    \includegraphics[scale = 0.7]{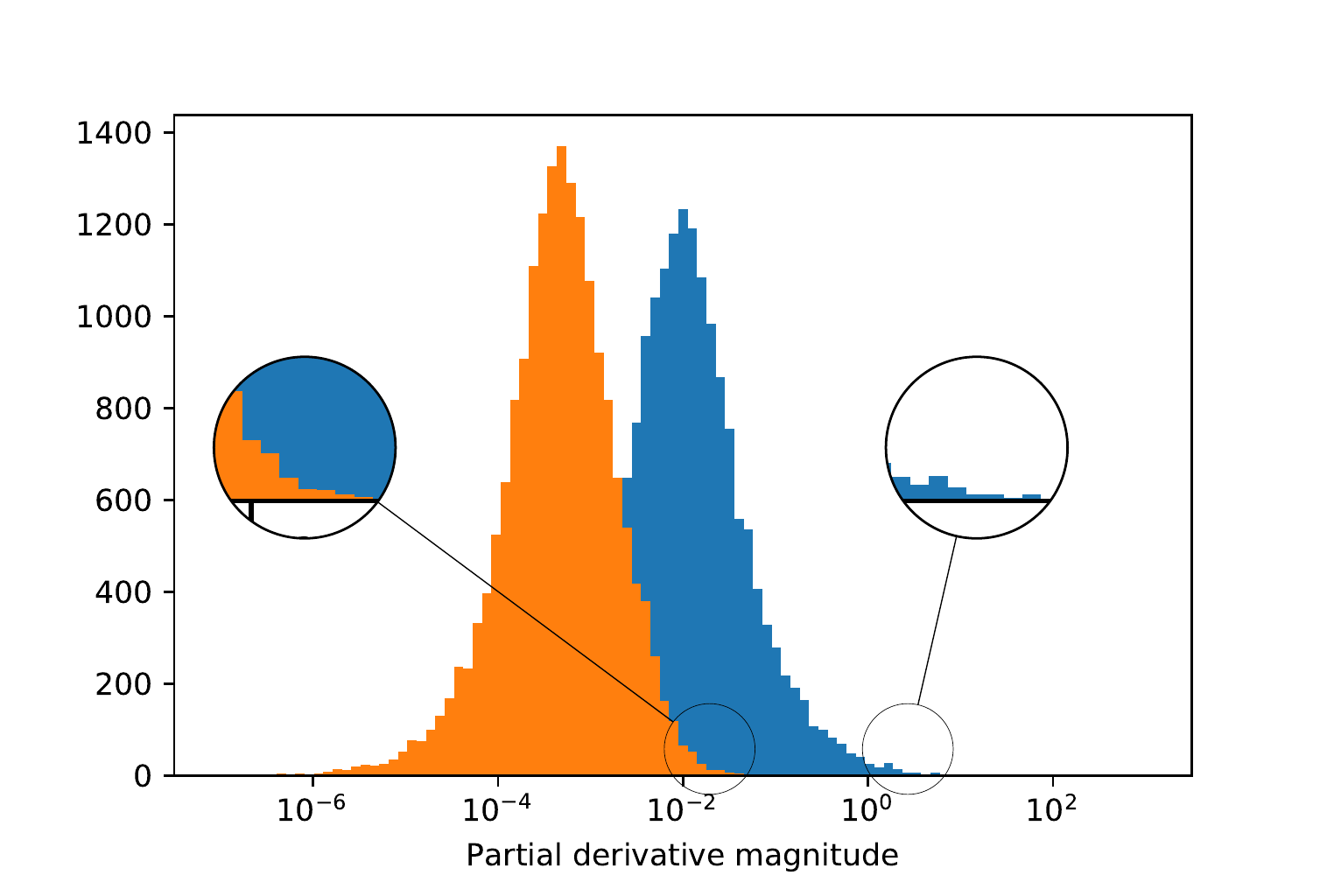}
    \caption{Histograms of the magnitude of the partial derivatives of the loss for the copying memory problem, both at initialization (blue) and after 500 iterations (orange). Note that a few partial derivatives are dominating by a couple of orders of magnitude. }
    \label{fig:partialCopyingGradient}
\end{figure}

\paragraph{Almost sparsity of the recurrent gradient.} Let us first illustrate that the gradient of the loss with respect to the orthogonal parameter (the recurrent weight matrix) has an approximately sparse representation in the basis \eqref{eq: basis}. Figure \ref{fig:partialCopyingGradient} illustrates the repartition of the Riemannian partial derivatives \eqref{eq: partial} (in absolute value, and computed over a minibatch), both at the initialization and after 500 training iterations. This figure indicates that a few partial derivatives have absolute value about two orders of magnitude larger than the bulk of partial derivatives. This observation supports the choice of the Gauss-Southwell coordinate selection strategy proposed in this paper. In this experiment, at initialization,  $0.1\%$ (resp. $4.2\%$) of the coordinates represent $95\%$ (resp. $99\%$) of the norm of the Riemannian gradient. After 500 iterations, $9\%$ (resp. $28.4\%$) of the coordinates represent $95\%$ (resp. $99\%$) of the norm of the Riemannian gradient.


\paragraph{Comparison of the algorithms.}
To illustrate further the benefits of our approach, we compare here SRGD with two variants of SRCD, in which the coordinate is selected at each iteration uniformly at random (SRCD-U) or using the Gauss-Southwell rule (SRCD-GS). For the sake of completeness, we also consider a block version of SRCD-GS, in which a few coordinates are selected at each iteration (the block size was here set to $0.5\%$ of the total number of coordinates in the tangent space).\footnote{Note that, in order for the update rule of the orthogonal parameter to be expressible in terms of Givens matrices, the selected coordinates have to correspond to disjoint pairs of columns. As this experiment simply aims to compare the methods in terms of accuracy, the matrix exponential was used to avoid here this constraint.} Though typically state-of-the-art algorithms addressing this problem rely on adaptive stepsizes, we compare here these algorithms using a fixed stepsize, so that the gap between the training loss of SRGD and SRCD-GS or SRCD-U gives us a measure of how harmful the restriction of the gradient to one/some coordinates is for the optimization. Figure \ref{fig:comparison} shows the evolution of the loss for the different algorithms, using a fixed stepsize set to $2 \cdot 10^{-4}$. Overall, our numerical experiments illustrate a good initial behavior of our proposed algorithms compared to SRGD, with a fast decrease of the loss over the first iterations. Though SRCD-GS outperforms SRCD-U over the first iterations, these two methods perform similarly over a larger number of iterations (the curve for SRCD-U has been averaged over ten runs). Updating at each iteration a block of coordinates improves significantly the performance, and provides a decrease of the loss that is very close to SRGD. Note also that, following our discussion on the approximately sparse representation of the stochastic Riemannian gradient in the basis \eqref{eq: basis}, it might be interesting to increase progressively the number of coordinates during the optimization; this is left for future research. 

\begin{figure}[t!]
    \centering
    \includegraphics[scale = 0.7,trim = 10 0 10 0, clip]{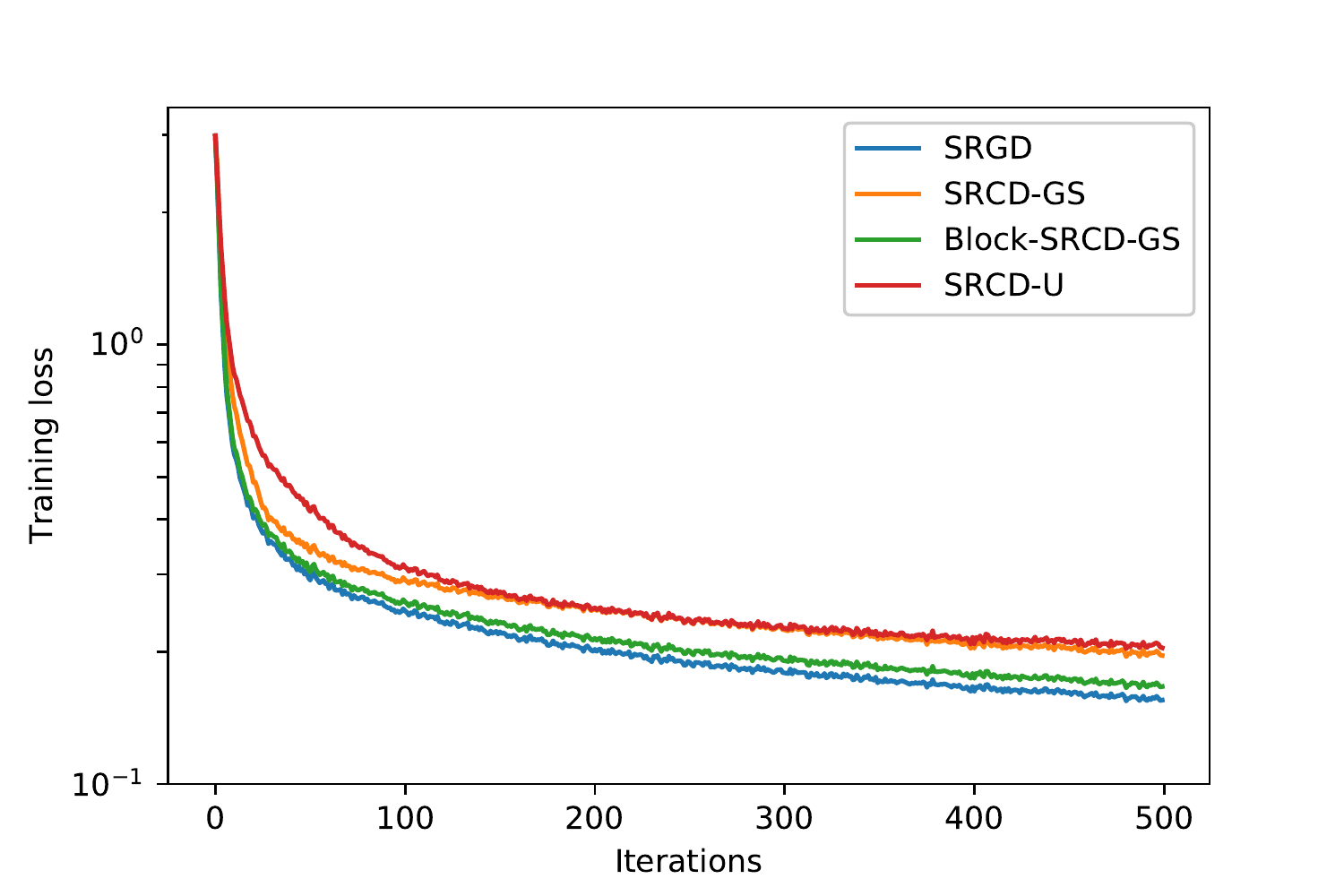}
    \caption{Comparison of stochastic Riemannian gradient descent (SRGD) and stochastic Riemannian coordinate descent (SRCD) on the copying memory problem. Note that a decrease of the loss comparable to SRGD can be obtained by restricting the update to one/some coordinates in the tangent space, resulting in a cost per iteration that evolves linearly with $d$, instead of cubically for SRGD.}
    \label{fig:comparison}
\end{figure}

\paragraph{Comparison in terms of computational cost.} Table~\ref{tab:time} presents the cost per iteration/update (with CUDA synchronization) using different optimizers for the copying memory problem (see supplementary material for experimental details). Though we would have expected SRCD to be significantly faster than SRGD, the computational cost of both methods is actually very close. Table~\ref{tab:time} indicates that the cost per iteration is indeed dominated by the cost of the backpropagation step (computing the Euclidean gradients with respect to the parameters) and that the cost of the parameter update is negligible in comparison. Actually, we didn't even notice any significant increase of computational time when imposing orthogonal constraints on the recurrent matrix via stochastic Riemannian gradient descent, compared to vanilla stochastic gradient descent.  Further experiments indicate that the cost of the matrix exponential evaluation and the cost of the backpropagation becomes more and more comparable as the size of the recurrent matrix increases. Hence, we argue that the most beneficial regime for the proposed algorithm is the very large dimensional regime, where the gap in accuracy per iteration between SRGD and SRCD (see Fig.~\ref{fig:comparison}) is expected to be compensated by the cost savings  per iteration. A detailed study is out of the scope of this paper and we defer it to future work.

\begin{table}[t!]
\centering
\caption{Average run-time cost per iteration/update of the iterates for the copying memory problem using the different optimizers considered in this paper}
\begin{tabular}{|c|c|c|c|c|}
\hline
\multicolumn{1}{|c|}{\multirow{2}{*}{Process}} & \multicolumn{4}{c|}{Optimizer} \\ \cline{2-5} 
\multicolumn{1}{|c|}{}  & SGD  & SRGD & SRCD-GS & SRCD-U \\ \hline
loss.backward() + optim.step() &   0.3125   &  0.3311 & 0.3191 &  0.3429 \\ \hline
optim.step()                &   0.00164   &  0.00197    &   0.00177  &  0.00223    \\ \hline
\end{tabular}
\label{tab:time}
\end{table}

\section{Conclusions}
We have proposed SRCD, a new algorithm for orthogonal RNN training with computational cost per iteration in $\mathcal{O}(d)$, in contrast with the $\mathcal{O}(d^3)$ cost of SRGD. We proved the convergence of our proposed algorithm under typical assumptions on the training problem (Lipshitz smoothness of the objective, classical assumptions on the gradient noise and stepsizes satisfying the Robbin-Monro conditions), for coordinates selected uniformly at random in $\{1, \dots, D\}$, with $D = d(d-1)/2$ the manifold dimension. We have also shown numerically that the Riemannian gradient has an approximately sparse representation in the basis \eqref{eq: basis}, and leveraged this observation by proposing a Gauss-Southwell coordinate selection rule. Numerical experiments indicate that the proposed rule leads to a faster initial decrease of the loss, compared to the uniform selection rule. Future research could aim to endow the proposed optimizer with an adaptive stepsize such as Adam or RMSProp, following the recent attempts \cite{Kasai2019, Becigneul2019, Lezcano2020} for developing adaptive stepsizes on manifolds.

\begin{ack}
The first author's work was supported by the National Physical Laboratory, Teddington, UK.
\end{ack}

\bibliographystyle{unsrt}
\bibliography{references}

\begin{thebibliography}{10}

\bibitem{Bengio1994}
Yoshua Bengio, Patrice Simard, and Paolo Frasconi.
\newblock Learning long-term dependencies with gradient descent is difficult.
\newblock {\em {IEEE Transactions on neural networks}}, 5(2):157 -- 166, 1994.

\bibitem{Pascanu2013}
Razvan Pascanu, Tomas Mikolov, and Yoshua Bengio.
\newblock On the difficulty of training recurrent neural networks.
\newblock In {\em 30th Internationl Conference on Machine Learning}, Atlanta,
  GA, USA, 2013.

\bibitem{Giles94}
C.~Lee Giles, Gary~M. Kuhn, and Ronald~J. Williams.
\newblock Dynamic recurrent neural networks: Theory and applications.
\newblock {\em IEEE Transactions on Neural Networks}, 5(2):153--156, 1994.

\bibitem{Arjovsky2016}
Martin Arjovsky, Amar Shah, and Yoshua Bengio.
\newblock Unitary evolution recurrent neural networks.
\newblock In {\em 33rd International Conference on Machine Learning (ICML)},
  New York, NY, USA, 2016.

\bibitem{Wisdom2016}
Scott Wisdom, Thomas Powers, John~R. Hershey, Jonathan~Le Roux, and Les Atlas.
\newblock Full-capacity unitary recurrent neural networks.
\newblock In {\em 30th Conference on Neural Information Processing Systems
  (NeurIPS)}, Barcelona, Spain, 2016.

\bibitem{Jing2017}
Li~Jing, Yichen Shen, Tena Dubcek, John Peurifoy, Scott Skirlo, Yann LeCun, Max
  Tegmark, and Marin Solja\v{c}i\'{c}.
\newblock {Tunable Efficient Unitary Neural Networks (EUNN) and their
  application to RNNs}.
\newblock In {\em 34th International Conference on Machine Learning (ICML)},
  Sydney, Australia, 2017.

\bibitem{Hyland2017}
Stephanie~L. Hyland and Gunnar R\"{a}tsch.
\newblock {Learning Unitary Operators with Help From u(n)}.
\newblock In {\em 31st AAAI Conference on Artificial Intelligence (AAAI)}, San
  Francisco, Calif., USA, 2017.

\bibitem{Mhammedi2017}
Zakaria Mhammedi, Andrew Hellicar, Ashfaqur Rahman, and James Bailey.
\newblock Efficient orthogonal parametrisation of recurrent neural networks
  using {Householder} reflections.
\newblock In {\em 34th International Conference on Machine Learning (ICML)},
  Sydney, Australia, 2017.

\bibitem{Helfrich2018}
Kyle~E. Helfrich, Devin Willmott, and Qiang Ye.
\newblock Orthogonal recurrent neural networks with scaled cayley transform.
\newblock In {\em 35th International Conference on Machine Learning (ICML)},
  Stockholm, Sweden, 2018.

\bibitem{Maduranga2019}
Kehelwala D.~G. Maduranga, Kyle~E. Helfrich, and Qiang Ye.
\newblock Complex unitary recurrent neural networks using scaled cayley
  transform.
\newblock In {\em 33rd AAAI Conference on Artificial Intelligence (AAAI)},
  Honolulu, TH, USA, 2019.

\bibitem{Lezcano2019}
Mario Lezcano-Casado and David Mart\'{i}nez-Rubio.
\newblock Cheap orthogonal constraints in neural networks: A simple
  parametrization of the orthogonal and unitary group.
\newblock In {\em 36th International Conference on Machine Learning (ICML)},
  Long Beach, Calif., USA, 2019.

\bibitem{Mctorch}
Mayank Meghwanshi, Pratik Jawanpuria, Anoop Kunchukuttan, Hiroyuki Kasai, and
  Bamdev Mishra.
\newblock {McTorch, a manifold optimization library for deep learning}.
\newblock arxiv preprint 1810.01811, 2018.

\bibitem{lezcano2019trivializations}
Mario Lezcano-Casado.
\newblock Trivializations for gradient-based optimization on manifolds.
\newblock In {\em 33rd Conference on Neural Information Processing Systems
  (Neurips)}, Vancouver, Canada, 2019.

\bibitem{Ozay2016}
Mete Ozay and Takayuki Okatani.
\newblock Optimization on submanifolds of convolution kernels in cnns.
\newblock arxiv preprint 1610.07008, 2016.

\bibitem{Huang2018}
Lei Huang, Xianglong Liu, Bo~Lang, Adams~Wei Yu, Yongliang Wang, and Bo~Li.
\newblock {Orthogonal Weight Normalization: Solution to Optimization over
  Multiple Dependent Stiefel Manifolds in Deep Neural Networks}.
\newblock In {\em 32nd AAAI Conference on Artificial Intelligence (AAAI)}, New
  Orleans, LA, USA, 2018.

\bibitem{Huang2020}
Lei Huang, Li~Liu, Fan Zhu, Diwen Wan, Zehuan Yuan, Bo~Li, and Ling Shao.
\newblock {Controllable Orthogonalization in Training DNNs}.
\newblock In {\em Conference on Computer Vision and Pattern Recognition
  (CVPR)}, virtual, 2020.

\bibitem{Bansal2018}
Nitin Bansal, Xiaohan Chen, and Zhangyang Wang.
\newblock Can we gain more from orthogonality regularizations in training deep
  {CNNs}?
\newblock In {\em 32nd Conference on Neural Information Processing Systems
  (NeurIPS)}, Montréal, Canada, 2018.

\bibitem{Jia2020}
Kui Jia, Shuai Li, Yuxin Wen, Tongliang Liu, and Dacheng Tao.
\newblock Orthogonal deep neural networks.
\newblock {\em {IEEE} Transactions on Pattern Analysis and Machine
  Intelligence}, 43(4):1352 -- 1368, 2020.

\bibitem{Pennington2017}
Jeffrey Pennington, Samuel~S. Schoenholz, and Surya Ganguli.
\newblock Resurrecting the sigmoid in deep learning through dynamical isometry:
  theory and practice.
\newblock In {\em 31st Conference on Neural Information Processing Systems
  (NIPS)}, Long Beach, Calif., USA, 2017.

\bibitem{Hu2020}
Wei Hu, Lechao Xiao, and Jeffrey Pennington.
\newblock Provable benefit of orthogonal initialization in optimizing deep
  linear networks.
\newblock In {\em 8th International Conference on Learning Representations
  (ICLR)}, virtual, 2020.

\bibitem{murray2021}
Michael Murray, Vinayak Abrol, and Jared Tanner.
\newblock Activation function design for deep networks: linearity and effective
  initialisation.
\newblock arxiv preprint 2105.07741, 2021.

\bibitem{Jia2017}
Kui Jia, Dacheng Tao, Shenghua Gao, and Xiangmin Xu.
\newblock Improving training of deep neural networks via singular value
  bounding.
\newblock In {\em Conference on Computer Vision and Pattern Recognition
  (CVPR)}, Honolulu, TH, USA, 2017.

\bibitem{Yoshida2017}
Yuichi Yoshida and Takeru Miyato.
\newblock Spectral norm regularization for improving the generalizability of
  deep learning.
\newblock arxiv preprint 1705.10941, 2017.

\bibitem{Wright2015}
S.~J. Wright.
\newblock Coordinate descent algorithms.
\newblock {\em Mathematical Programming}, 151:3--–34, 2015.

\bibitem{Zheng2019}
J.~Zeng, T.~T.-K. Lau, S.~Lin, and Y.~Yao.
\newblock Global convergence of block coordinate descent in deep learning.
\newblock In {\em 36th International Conference on Machine Learning (ICML)},
  Long Beach, Calif., USA, 2019.

\bibitem{Palagi2019}
L.~Palagi and R.~Seccia.
\newblock Block layer decomposition schemes for training deep neural networks.
\newblock {\em Journal of Global Optimization}, 77:97--124, 2019.

\bibitem{Patrascu2014}
A.~Patrascu and I.~Necoara.
\newblock Efficient random coordinate descent algorithms for large-scale
  structured nonconvex optimization.
\newblock {\em Journal of Global Optimization}, 61(1):19--46, 2014.

\bibitem{Wang2014}
H.~Wang and A.~Banerjee.
\newblock Randomized block coordinate descent for online and stochastic
  optimization.
\newblock arxiv preprint 1407.0107, 2014.

\bibitem{Zhao2014}
T.~Zhao, M.~Yu, Y.~Wang, R.~Arora, and H.~Liu.
\newblock Accelerated mini-batch randomized block coordinate descent method.
\newblock In {\em 28th Conference on Neural Information Processing Systems
  (NeurIPS)}, Montr\'{e}al, Canada, 2014.

\bibitem{Palagi2021}
L.~Palagi and R.~Seccia.
\newblock On the convergence of a block-coordinate incremental gradient method.
\newblock {\em Soft Computing}, 2021.

\bibitem{Ishteva2013}
M.~Ishteva, P.-A. Absil, and P.~Van Dooren.
\newblock Jacobi algorithm for the best low multilinear rank approximation of
  symmetric tensors.
\newblock {\em {SIAM} Journal on Matrix Analysis and Applications},
  34(2):651--672, 2013.

\bibitem{Shalit2014}
Uri Shalit and Gal Chechik.
\newblock Coordinate-descent for learning orthogonal matrices through givens
  rotations.
\newblock In {\em 31st International Conference on Machine Learning (ICML)},
  Beijing, China, 2014.

\bibitem{Gutman2020}
David~Huckleberry Gutman and Nam Ho-Nguyen.
\newblock Coordinate descent without coordinates: Tangent subspace descent on
  {Riemannian} manifolds.
\newblock arxiv preprint 1912.10627, 2020.

\bibitem{AMS2008}
P.-A. Absil, R.~Mahony, and R.~Sepulchre.
\newblock {\em Optimization Algorithms on Matrix Manifolds}.
\newblock Princeton University Press, 2008.

\bibitem{boumal2020intromanifolds}
Nicolas Boumal.
\newblock An introduction to optimization on smooth manifolds.
\newblock Available online, Nov 2020.

\bibitem{Bonnabel2013}
S.~Bonnabel.
\newblock Stochastic gradient descent on {Riemannian} manifolds.
\newblock {\em IEEE Transactions Automatic Control}, 58(9):2217--2229, 2013.

\bibitem{Bader2019}
P.~Bader, S.~Blanes, and F.~Casas.
\newblock Computing the matrix exponential with an optimized taylor polynomial
  approximation.
\newblock {\em Mathematics}, 7(12), 2019.

\bibitem{Bottou2018}
L.~Bottou, F.~E. Curtis, and J.~Nocedal.
\newblock Optimization methods for large-scale machine learning.
\newblock {\em {SIAM} Review}, 60(2):223--311, 2018.

\bibitem{Boumal2018}
N.~Boumal, P.-A. Absil, and C.~Cartis.
\newblock Global rates of convergence for nonconvex optimization on manifolds.
\newblock {\em {IMA} Journal of Numerical Analysis}, 39(1):1--33, 2018.

\bibitem{Robbins1951}
H.~Robbins and S.~Monro.
\newblock A stochastic approximation method.
\newblock {\em The Annals of Mathematical Statistics}, 22(3):400--407, 1951.

\bibitem{Kasai2019}
Hiroyuki Kasai, Pratik Jawanpuria, and Bamdev Mishra.
\newblock Riemannian adaptive stochastic gradient algorithms on matrix
  manifolds.
\newblock In {\em 36th International Conference on Machine Learning (ICML)},
  Long Beach, Calif., USA, 2019.

\bibitem{Becigneul2019}
G.~B\'{e}cigneul and O.-E. Ganea.
\newblock Riemannian adaptive optimization methods.
\newblock In {\em 7th International Conference on Learning Representations},
  New Orleans, LA, 2019, 2019.

\bibitem{Lezcano2020}
Mario Lezcano-Casado.
\newblock Adaptive and momentum methods on manifolds through trivializations.
\newblock arxiv preprint 2010.04617, 2020.

\bibitem{Henaff2016}
M.~Henaff, A.~Szlam, and Y.~LeCun.
\newblock Recurrent orthogonal networks and long-memory tasks.
\newblock In {\em {33rd International Conference on Machine Learning (ICML)}},
  New York, NY, USA, 2016.

\bibitem{He2015}
Kaiming He, Xiangyu Zhang, Shaoqing Ren, and Jian Sun.
\newblock Delving deep into rectifiers: Surpassing human-level performance on
  {ImageNet} classification.
\newblock In {\em IEEE International Conference on Computer Vision (ICCV)},
  Santiago, Chili, 2015.

\end{thebibliography}

\section{Appendices}
\subsection{Proof of Theorem 1}
Our proof heavily relies on the proof of \cite[Theorem 4.10]{Bottou2018}, namely, the convergence proof of stochastic gradient descent for nonconvex objectives, in the Euclidean setting. Let us write $f^k = f(X^k, W^k)$, $\nu^k = - \alpha^k \tilde{g}_X^k$ and $\mu^k = - \alpha^k \tilde{g}_{W,i^k}^k$. Under Assumption 1, there holds:
\begin{equation*}
\begin{aligned}
    f^{k+1} = f(X^k+\nu^k,\Exp_{W^{k}}(\mu^k)) &\leq  f^k - \alpha^k \langle g_X^k, \tilde{g}_X^k  \rangle - \alpha^k \langle g_W^k, \tilde{g}_{W,i^k}^k  \rangle 
    + \frac{L}{2} (\alpha^k)^2 \left( \| \tilde{g}_X^k \|^2 + \| \tilde{g}_{W,i^k}^k \|^2 \right).
    \end{aligned}
\end{equation*}
Taking the conditional expectation on both sides, and using the fact that $f^k$ and the exact gradients $g_X^k$ and $g_W^k$ are $\calF^k$-measurable, we get:
\begin{equation*}
\begin{aligned}
    \E \left[ f^{k+1} | \calF^k \right] &\leq  f^k - \alpha^k \E \left[  \langle g_X^k, \tilde{g}_X^k  \rangle | \calF^k \right] - \alpha^k \E \left[ \langle g_W^k, \tilde{g}_{W,i^k}^k  \rangle | \calF^k \right]\\
    &+ \frac{L}{2} (\alpha^k)^2 \E \left[  \| \tilde{g}_X^k \|^2 + \| \tilde{g}_{W,i^k}^k \|^2 | \calF^k \right].
    \end{aligned}
\end{equation*}
Though some of the terms of the right-hand side of this last inequality can be directly lower bounded using Assumption 2, note that the conditional expectations of $\langle g_W^k, \tilde{g}_{W,i^k}^k  \rangle$ and $ \| \tilde{g}_{W,i^k}^k \|^2$, involve jointly randomness in the gradient and in the coordinate selection. Since the coordinate $i^k$ is selected uniformly at random among $\{1, \dots, D\}$, there follows:
\begin{equation*}
\begin{aligned}
    \E \left[ \langle g_W^k, \tilde{g}_{W,i^k}^k  \rangle  | \calF^k \right] 
    = \frac{1}{D}  \E \left[ \sum_{i = 1}^D \langle \tilde g_W^k, \eta_i \rangle \langle g_W^k, \eta_i  \rangle  | \calF^k \right] = \frac{1}{D}  \E \left[ \langle \tilde{g}_W^k, g_W^k \rangle  | \calF^k \right] \geq \frac{1}{D} \mu_W \|g_{W}^k \|^2.
    \end{aligned}
\end{equation*}
where we used Assumption 2 and the fact that the set of vectors $\{\eta_i\}$ with $i \in \{1, \dots, D\}$ is an orthonormal basis of $\mathcal{T}_{W^k} \O_d$. Similarly, 
\begin{equation*}
\begin{aligned}
    \E \left[ \| \tilde{g}_{W,i^k}^k \|^2|\calF^{k}\right]  = \frac{1}{D}  \E \left[ \sum_{i  = 1}^D \langle \tilde{g}_{W}^k, \eta_i\rangle^2 | \calF^k \right] = \frac{1}{D}  \E \left[ \|\tilde{g}_{W}^k \|^2 | \calF^k \right] \leq \frac{C_W}{D} + \frac{M_W}{D} \|g_{W}^k \|^2.
    \end{aligned}
\end{equation*}
Defining
\[\mu := \min \{\mu_X, \mu_W/D\}, \quad \ C:= C_X+ \frac{C_W}{D}, \quad M := \max \{M_X, M_W/D\}, \]
there follows
\[ \E \left[ f^{k+1} | \calF^k \right] \leq  f^k - \alpha^k (\mu - \alpha^k \frac{L}{2} M)  \| g^k \|^2 + \frac{L}{2} (\alpha^k)^2 C. \]
Note that, since $\alpha^1, \alpha^2, \dots$ is a decreasing sequence, one can without loss of generality assume that $\alpha^k L M \leq \mu$ so that 
\[  \E \left[ f^{k+1} | \calF^k \right] \leq   f^k - \alpha^k \frac{\mu}{2}  \| g^k \|^2 + \frac{L}{2} (\alpha^k)^2 C . \]
Slightly rewriting this, we get:
\[ f^k - \E \left[f^{k+1}|\calF^{k}\right] \geq  \left( \alpha^k \frac{\mu}{2} \| g^k \|^2 - \frac{L}{2} (\alpha^k)^2 C \right). \]
Taking the total expectation of the previous inequality yields: 
\[  \E \left[ f^k \right] - \E \left[f^{k+1}\right] \geq  \left( \alpha^k \frac{\mu}{2} \E \left[ \| g^k \|^2 \right] - \frac{L}{2} (\alpha^k)^2 C \right). \]
We then sum both sides of this inequality over $k \in \{0, \dots, K\}$ and used Assumption 2.1 to obtain 
\[ \E \left[f^0\right] - f_{\inf} \geq  \E \left[ f^0 \right] - \E \left[f^{K+1}\right] \geq \sum_{k = 0}^K \left( \alpha^k \frac{\mu}{2} \E \left[ \| g^k \|^2 \right] - \frac{L}{2} (\alpha^k)^2 C \right), \]
which yields
\[ \sum_{k = 0}^K \alpha^k \E\left[\| g^k\|^2 \right] \leq 2 \left(\frac{\E\left[ f^0\right] - f_{\inf}}{\mu}\right) +  \frac{LC}{\mu} \sum_{k = 0}^K(\alpha^k)^2. \]
The Robbin-Monro conditions imply then that 
\[ \lim_{K \to \infty} \sum_{k = 0}^K \alpha^k \E\left[\| g^k\|^2 \right] < \infty. \]
The claim simply follows from the fact that $\lim_{K \to \infty} \sum_{k = 0}^K \alpha^k = \infty$.
\vspace{2em}

\subsection{Experimental setting for the numerics}
\paragraph{Initialization and structure of the network: } The network architecture is described by the following signal propagation equations:

\begin{equation} \label{eq:network}
    \left\{ \begin{array}{rcl}
    h(t+1) & = &\phi(W_{\input} x(t+1) + W h(t)), \\ 
    y(t+1) & = & W_{\output} h(t+1) + b_{\output}.
    \end{array} \right.
\end{equation}

The recurrent matrix $\nbf{W}$ is initialized as proposed by Henaff et al \cite{Henaff2016} and used in, e.g., \cite{Lezcano2019}: 
\[ \nbf{W}^0 = (\nbf{I} + \nbf{A})^{-1} (\nbf{I} - \nbf{A}), \]
the Cayley transform of a block-diagonal matrix $\nbf{A}$, whose diagonal blocks are of the form
\[  \begin{pmatrix}
0 & s \\
-s & 0
\end{pmatrix},\]
with $s$ sampled independently for each block over the uniform distribution $\mathcal{U}[-\pi, \pi]$. Following the experimental setting of \cite{Lezcano2019}, we choose the modrelu activation proposed in \cite{Arjovsky2016}. The input-to-hidden and hidden-to-output weight matrices $W_{\input}$ and $W_{\output}$ are initialized using the He initialization \cite{He2015}, the components of the $b$ parameter characterizing the modrelu are initially randomly sampled from the uniform distribution $\mathcal{U}([-0.01, 0.01])$, and the bias vector $b_{\output}$ of the output layer is initialized at zero. 

\paragraph{Runtime benchmarking of the optimizers: }
In order to benchmark the different optimizers considered in this study, we ran experiments on a network with recurrent matrix of size $10^3\times 10^3$ trained with a batch-size of $128$ on a workstation configured with Intel(R) Xeon(R) 4114 CPU and P100 GPU. The results are averaged over 50 trials with CUDA(R) synchronization and 1000 iterations per trial on the copying memory problem.

\end{document}